\documentclass{article}

\usepackage{fullpage,amsmath,amsthm,amssymb,graphicx,pifont,url,xcolor,subcaption}

\def\Var{\mathrm{Var}}
\def\dt{\mathrm{d}t}
\def\erf{\mathrm{erf}}
\def\dmu{\mathrm{d}\mu}
\def\du{\mathrm{d}u}
\def\TV{\mathrm{TV}}
\def\KL{\mathrm{KL}}
\def\JS{\mathrm{JS}}
\def\calI{\mathcal{I}}
\def\calX{\mathcal{X}}
\def\calF{\mathcal{F}}
\renewcommand{\Re}{\mathbb{R}}
\def\eqdef{{:=}}

\newtheorem{theorem}{Theorem}

\begin{document}

\title{Guaranteed Deterministic Bounds on the Total Variation Distance between Univariate Mixtures}

\author{Frank Nielsen\\Sony Computer Science Laboratories, Inc.\\Japan\\{\small\tt{}Frank.Nielsen@acm.org}%
\and%
Ke Sun\\Data61\\Australia\\{\small\tt{}Ke.Sun@data61.csiro.au}}
\date{}
\maketitle

\begin{abstract}
The total variation distance is a core statistical distance between probability measures that satisfies the metric axioms, with value always  falling in  $[0,1]$.
This distance plays a fundamental role in machine learning and signal processing:
It is a member of the broader class of $f$-divergences, and it is related to the probability of error in Bayesian hypothesis testing. 
Since the total variation distance does not admit closed-form expressions for statistical mixtures (like Gaussian mixture models),
one often has to rely in practice on costly numerical integrations or on fast Monte Carlo approximations that however do not guarantee deterministic lower and upper bounds.
In this work, we consider two methods for bounding the total variation of univariate mixture models:
The first method is based on the information monotonicity property of the total variation to design guaranteed nested deterministic lower bounds.
The second method relies on computing the geometric lower and upper envelopes of weighted mixture components to derive deterministic bounds based on density ratio. 
We demonstrate the tightness of our bounds in a series of experiments on Gaussian, Gamma and Rayleigh mixture models.
\end{abstract}

%%%%%%%%
\section{Introduction}\label{sec:intro}
%%%%%%%%%

%%
\subsection{Total variation and $f$-di\-ver\-gences}
%%%%%

Let $(\calX\subset \Re,\calF)$ be a measurable space on the sample space $\calX$ equipped with the Borel $\sigma$-algebra~\cite{Ross-2014},
and $P$ and $Q$ be two probability measures with respective densities $p$ and $q$ with respect to the Lebesgue measure $\mu$. 
The Total Variation distance~\cite{Csiszar-2004} (TV for short) is a statistical {\em metric} distance defined by
\begin{equation*}
\TV(P,Q) \,\eqdef\, \sup_{E\in\calF}|P(E)-Q(E)| = \TV(p,q),
\end{equation*}
with
\begin{equation*}
\TV(p,q) = \frac{1}{2}\int_{\calX} |p(x)-q(x)|\dmu(x)
          = \frac{1}{2}\,\|p(x)-q(x)\|_1.
\end{equation*}
The TV distance ranges in $[0,1]$, and is related to the probability of error $P_e$ in Bayesian statistical hypothesis testing~\cite{GenPe-2014},
so that $P_e(p,q)=\frac{1}{2}(1-\TV(p,q))$. Since we have for any $a, b\in\Re^+$
\begin{equation*}
\frac{1}{2}\vert{}a-b\vert{}=\frac{a+b}{2}-\min(a,b)=\max(a,b)-\frac{a+b}{2},
\end{equation*}
we can rewrite the TV equivalently as
\begin{align}
\TV(p,q)
&=\int_{\calX} \left(\frac{p(x)+q(x)}{2}-\min(p(x),q(x))\right)\dmu(x),\nonumber\\
&= 1-\int_{\calX} \min(p(x),q(x))\dmu(x)= \int_{\calX} \max(p(x),q(x))\dmu(x) - 1.\label{eq:tvminmax}
\end{align}

Thus by bounding the ``histogram similarity''~\cite{ColorIndexing-1991,GenPe-2014} 
\begin{equation*}
h(p,q)\;\eqdef\,\int_{\calX} \min(p(x),q(x))\dmu(x),
\end{equation*}
or equivalently 
\begin{equation*}
H(p,q)\;\eqdef\,\int_{\calX} \max(p(x),q(x))\dmu(x),
\end{equation*}
since $\int_{\calX} (\max(p(x),q(x))+\min(p(x),q(x)) ) \dmu(x)=2$ by eq.~(\ref{eq:tvminmax}),
we obtain corresponding bounds for the TV and Bayes' error probability $P_e$.

%%%%%%%%%%
\subsection{Prior work}
%%%%%%%%%%

For simple univariate distributions like univariate Gaussian distributions,
the TV may admit a closed-form expression.
For example, consider exponential family distributions~\cite{EF-2009}
with density $p(x;\theta_p)=\exp(\theta_p^\top t(x) - F(\theta_p))$
and $q(x)=p(x;\theta_q)=\exp(\theta_q^\top t(x) - F(\theta_q))$.
When we can compute {\em exactly} the root solutions of
$(\theta_p-\theta_q)^\top t(x)=F(\theta_p)-F(\theta_q)$, e.g., $t(x)$ encode a polynomial of degree at most $5$,
then we can split the distribution support as $\Re=\uplus_{s=1}^{l} I_s$ based on the roots.
Then in each interval $I_s$, we can compute the elementary interval integral using the cumulative distribution
functions $\Phi_p$ and $\Phi_q$.
Indeed, assume without loss of generality that $p(x)-q(x)\geq 0$ on an interval $I=(a,b)$. 
Then we have
\begin{equation*}
\frac{1}{2}\int_I \vert p(x)-q(x) \vert \dmu(x)= \frac{1}{2}\left(\Phi_p(b)-\Phi_p(a)-\Phi_q(b)+\Phi_q(a)\right).
\end{equation*}
For univariate Gaussian distributions $p_1(x)=p(x;\mu_1,\sigma_1)$ and $p_2(x)=p(x;\mu_2,\sigma_2)$ with $\sigma_1\not=\sigma_2$ (and $t(x)=(x,x^2)$),
the quadratic equation $(\theta_p-\theta_q)^\top t(x)=F(\theta_p)-F(\theta_q)$ expands as $ax^2+bx+c=0$,
where 
\begin{align*}
a &= \frac{1}{\sigma_1^2}-\frac{1}{\sigma_2^2},\\
b &= 2\left(\frac{\mu_2}{\sigma_2^2}-\frac{\mu_1}{\sigma_1^2}\right),\\
c &= \left(\frac{\mu_1}{\sigma_1}\right)^2 - \left(\frac{\mu_2}{\sigma_2}\right)^2 + 2\log \frac{\sigma_1}{\sigma_2}.
\end{align*}
We have two distinct roots $x_1=\frac{-b-\sqrt{\Delta}}{2a}$ and $x_2=\frac{-b+\sqrt{\Delta}}{2a}$
with $\Delta=b^2-4ac\geq 0$.
Therefore the TV between univariate Gaussians writes as follows:
\begin{equation*} 
\TV(p_1, p_2) = \frac{1}{2}
\left|
\erf\left(\frac{x_1-\mu_1}{\sigma_1\sqrt{2}}\right)-\erf\left(\frac{x_1-\mu_2}{\sigma_2\sqrt{2}}\right)
\right| +
\frac{1}{2} \left|
\erf\left(\frac{x_2-\mu_1}{\sigma_1\sqrt{2}}\right)-\erf\left(\frac{x_2-\mu_2}{\sigma_2\sqrt{2}}\right)
\right|,
\end{equation*}
where  $\erf(x)=\frac{1}{\sqrt{\pi}} \int_{-x}^x e^{-t^2} \dt$ denotes the error function.
Notice that it is difficult problem to bound or find the modes of a GMM~\cite{ModeGMM-2017},
and therefore to decompose the TV between GMMs into elementary intervals.

In practice, for mixture models (like Gaussian mixture models, GMMs),
the TV is approximated by either \ding{192} discretizing the integral (i.e., numerical integration)
\begin{equation*}
\widetilde{\TV}_m(p,q)=\frac{1}{2} \sum_{i=1}^m |p(x_i)-q(x_i)| (x_{i+1}-x_i) \geq 0,
\end{equation*}
($\widetilde{\TV}_m(p,q)\simeq \TV(p,q)$) for $x_1<\ldots<x_{m+1}$ (one can choose any quadrature rule)
or \ding{193} performing stochastic Monte Carlo (MC) integration via importance sampling: 
\begin{equation*}
\widehat{\TV}_m(p,q)=\frac{1}{2m} \sum_{i=1}^m \frac{1}{r(x_i)}|p(x_i)-q(x_i)| \geq 0,
\end{equation*}
where $x_1,\ldots, x_m$ are independently and identically distributed (iid) samples from a proposal distribution $r(x)$.
Choosing $r(x)=p(x)$ yields
\begin{equation*}
\widehat{\TV}_m(p,q)=\frac{1}{2m} \sum_{i=1}^m  \left|1-\frac{q(x_i)}{p(x_i)}\right| \geq 0,
\end{equation*}

While \ding{192} is time consuming, \ding{193} cannot guarantee deterministic bounds although it is asymptotically a consistent estimator (confidence intervals can be calculated):
$\lim_{m\rightarrow\infty} \widehat{\TV}_m(p,q)=\TV(p,q)$ (provided that the variance $\Var_p[\frac{q(x)}{p(x)}]<\infty$).
This raises the problem of consistent calculations, since for $p,q,r$, we may have a first run
with $\widehat{\TV}_m(p,q)>\widehat{\TV}_m(p,r)$, and second run with $\widehat{\TV}_m(p,q)<\widehat{\TV}_m(p,r)$.
Thus we seek for guaranteed deterministic lower (L) and upper (U) bounds so that $L(p,q)\leq \TV(p,q)\leq U(p,q)$.

The TV is the only metric {\em $f$-divergence}~\cite{fdiv-2007} so that
\begin{align*}
\TV(p,q) &= I_{f_\TV}(p:q),
\end{align*}
where
\begin{align*}
I_f(p:q) &\eqdef  \int_{\calX} p(x)f\left(\frac{q(x)}{p(x)}\right) \dmu(x),\\
f_\TV(u) &= \frac{1}{2}|u-1|.
\end{align*}
A $f$-divergence can either be bounded (e.g., TV or the Jensen-Shannon divergence) or unbounded when the integral diverges (e.g., the Kullback-Leibler divergence or the $\alpha$-divergences~\cite{IG-2016}).

Consider two finite mixtures $m(x)=\sum_{i=1}^k w_i p_i(x)$ and $m'(x)=\sum_{j=1}^{k'} w_{j}' p_j'(x)$.
Since $I_f(\cdot:\cdot)$ is jointly convex, we have 
$I_f(m:m')\leq \sum_{i,j}^{k,k'} w_i w'_{j} I_f(p_i(x):p_{j}'(x))$.
We may also refine this upper bound by using a variational bound~\cite{Hershey-2007,KLGMM-2012}.
However,  these upper bounds are too loose for TV as they can easily go above the trivial upper bound of $1$. 

In information theory, Pinsker's inequality relates the Kullback-Leibler divergence to the TV by
\begin{equation}\label{eq:pinsker}
\KL(p:q)\geq (2\log e)\;\TV^2(p:q).
\end{equation}

Thus we can upper bound TV in term of KL as follows: $\TV(p:q)\leq \sqrt{\frac{1}{2\log e}\KL(p:q)}$.
Similarly, we can upper bound TV using any $f$-divergences~\cite{GenPinsker-2009}.
However, the bounds may be implicit because the paper~\cite{GenPinsker-2009} considered 
the best lower bounds of a given $f$-divergence in term of total variation.
For example, it is shown (\cite{GenPinsker-2009}, p. 15) that the Jensen-Shannon divergence is lower bounded by 
\begin{align*}
\JS(p:q)
&\ge\left(\frac{1}{2}-\frac{\TV(p:q)}{4}\right)\log(2-\TV(p:q))\\
&+\left(\frac{1}{2}+\frac{\TV(p:q)}{4}\right)\log(2+\TV(p:q))-\log 2.
\end{align*}
See also~\cite{reversePinsker-2015} for reverse Pinsker inequalities (introducing crucial ``fatness conditions'' on the distributions since otherwise the $f$-divergences may be unbounded).
We may then apply combinatorial lower and upper bounds on $f$-divergences of mixture models, following the method of~\cite{NS-2016}, to get bounds on TV.
However, 
it is challenging to have our bounds for the TV $f$-divergence beat the naive upper bound of $1$  and the lower bound of $0$.

%%%%%%%%%%%%%%%
\subsection{Contributions and paper outline}
%%%%%%%%%%%%%%%

We summarize our main contributions as follows:
\begin{itemize}
	\item We describe the Coarse-Grained Quantized Lower Bound (CGQLB, Theorem~1 in \S\ref{sec:TVbounds}) by proving the information monotonicity of the total variation distance.
	\item We present the Combinatorial Envelope Lower and Upper bounds (CELB/CEUB, Theorem~2 in \S\ref{sec:geoenv}) 
	for the TV between univariate mixtures that rely on geometric envelopes and density ratio bounds.
\end{itemize} 

The paper is organized as follows:
We present our deterministic bounds in \S\ref{sec:TVbounds} and in \S\ref{sec:geoenv}.
We demonstrate numerical simulations in \S\ref{sec:exp}.
Finally, \S\ref{sec:concl} concludes and hints at further perspectives for designing bounds on $f$-divergences.

%%%%% 
\section{TV bounds from information monotonicity}\label{sec:TVbounds}
%%%%%%
 
Let us prove the {\em information monotonicity} property~\cite{IG-2016} of the total variation distance:
coarse-graining  the (mixture) distributions necessarily decreases their total variation.\footnote{This is not true for the Euclidean distance.}

Let $\calI=\uplus_{s=1}^{l} I_s$ be an arbitrary finite partition of the support $\calX$ into $l$ intervals.
Using the cumulative distribution functions (CDFs) of mixture components, we can calculate the mass of mixtures inside each elementary interval
as a weighted sum of the component CDFs.
Let $m_\calI$ and $m'_\calI$ denote the induced coarse-grained discrete distributions (also called lumping~\cite{Csiszar-2004}).
Their total variation distance is
\begin{equation*}
\TV(m_\calI,m'_\calI)
\eqdef \frac{1}{2} \sum_{s=1}^l \left|m_\calI^s-{m'}_\calI^s\right|.
\end{equation*}

\begin{theorem}[Information monotonicity of TV]
The information monotonicity of the total variation ensures that
\begin{equation}
0\leq \TV(m_\calI,m'_\calI)\leq \TV(m,m') \leq 1.
\end{equation}
\end{theorem}

\begin{proof}
\begin{equation*}
\TV(m,m') = \int \max(m(x),m'(x)) \dmu(x) -1 = \sum_{s=1}^l \int_{I_s} \max(m(x),m'(x)) \dmu(x) -1.
\end{equation*}
Since $\int_{I_s} \max(m(x),m'(x)) \dmu(x) \ge \int_{I_s} m(x) \dmu(x) $
and $\int_{I_s} \max(m(x),m'(x)) \dmu(x) \ge \int_{I_s} m'(x) \dmu(x) $, we have
\begin{equation*}
\int_{I_s} \max(m(x),m'(x)) \dmu(x) \ge
\max\left(\int_{I_s} m(x) \dmu(x), \int_{I_s} m'(x) \dmu(x)\right).
\end{equation*}
Therefore
\begin{align*}
\TV(m,m') 
&\ge \sum_{s=1}^l \max\left(\int_{I_s}m(x)\dmu(x),\int_{I_s}m'(x)\dmu(x)\right)-1\\
&=\sum_{s=1}^l \max\left(m_\calI^s, m_\calI'^s\right)-1\\
&= \TV(m_\calI,m'_\calI).
\end{align*}
%since $\max(a_1,b_1)+\max(a_2,b_2) \geq \max(a_1+a_2,b_1+b_2)$ for any numbers $a_1, a_2, b_1, b_2$.
%Indeed, wlog, assume $a_2\geq b_2$ then $\max(a_1,b_1)+\max(a_2,b_2)=
%\max(a_1+a_2,b_1+a_2)\geq \max(a_1+a_2,b_1+b_2)$.
\end{proof}

Note that we coarse-grain a continuum support $\Re$ (or $\Re^+$, say for Rayleigh mixtures) into a finite number of bins.
The proof does not use the fact that the support is 1D and is therefore generalizable to the multi-dimension case.
For the discrete case, the proof~\cite{IG-2016} will be different.
In summary, this approach yields the Coarse-Grained Quantization Lower bound (CGQLB).

By creating a hierarchy of $h$ nested partitions $\calI_h\subset\ldots\subset\calI_1\subset \calI_0=\calX$, we get the telescopic inequality:
\begin{equation*}
\TV(m_{\calI_h},m'_{\calI_h})\leq \ldots\leq \TV(m_{\calI_1},m'_{\calI_1})\leq \TV(m,m').
\end{equation*}

This coarse-graining technique yields lower bounds for any $f$-divergence due to their information monotonicity property~\cite{Csiszar-2004}.

%%%%
%\subsection{Upper bound: Particular case of shared components}
%%%%%

We present now a simple upper bound when dealing with a very specific case of mixtures.
Consider mixtures sharing the {\em same} prescribed components (i.e., only weights may differ).
For example, this scenario occurs when we {\em jointly} learn a set of mixtures from several datasets~\cite{comix-2016}.
Then it comes that
\begin{align}\label{eq:comix}
\TV(m,m') & =\frac{1}{2} \int \left| \sum_{i=1}^k (w_i-w_i') p_i(x)  \right| \dmu(x)\nonumber\\
&\le \frac{1}{2} \int \sum_{i=1}^k \left\vert w_i-w_i'\right\vert p_i(x) \dmu(x)\nonumber\\
&=\frac{1}{2} \sum_{i=1}^k \left\vert w_i-w_i'\right\vert \int p_i(x) \dmu(x)\nonumber\\
&=\frac{1}{2} \sum_{i=1}^k |w_i-w_i'| \leq 1.
\end{align}
We may always consider mixtures $m$ and $m'$ sharing the same $k+k'$ prescribed components (by allowing some weights to be zero).
In that case, let $w$ and $w'$ denote the common weight distribution. From the above derivations we get $\TV(m,m')\leq \TV(w,w')$.
However, when mixtures do not share components, we end up the trivial upper bound of $1$ since in that case $\sum_{i=1}^k |w_i-w_i'|=2$.
The upper bound in eq.~(\ref{eq:comix}) can be easily extended to mixture of positive measures (with weight vectors not necessarily normalized to one).

%%%%%%%%%%%%
\section{TV bounds via geometric envelopes}\label{sec:geoenv}
%%%%%%%%%%%%

Consider two statistical mixtures $m(x)=\sum_{i=1}^k w_i p_i(x)$ and $m'(x)=\sum_{i=1}^{k'} w_i' p_i'(x)$.
Let us bound $h(m,m')$ following the computational geometric technique introduced in~\cite{NS-2016,NS-2017} as follows
\begin{align}
p_{l(x)}   &\leq m(x)  \leq p_{u(x)},\nonumber\\
p'_{l'(x)} &\leq m'(x)  \leq p'_{u'(x)},\nonumber
\end{align}
where $l(x)$ and $u(x)\in [k]=\{1,\ldots, k\}$ and $l'(x)$ and $u'(x)\in [k']$ denote respectively the {\em indices} of the component of mixture $m(x)$ (resp. $m'(x)$) 
that is the lowest (resp. highest) at position $x\in\calX$.
The sequences of $l(x),u(x),l'(x),u'(x)$ are piecewisely integer constant when $x$ swipes through $\calX$,
and can be computed from lower and upper geometric envelopes of the mixture component probability distributions~\cite{NS-2016,NS-2017}.
It follows that 
\begin{equation}\label{eq:lub}
\min(p_{l(x)},p'_{l'(x)}) \leq \min(m(x),m'(x)) \leq \max(p_{u(x)},p'_{u'(x)}).
\end{equation}
We partition the support $\calX$ into $l=O_{k,k'}(1)$ elementary intervals
$I_1=(a_1,b_1), [a_2, b_2), \ldots, I_s=[a_l,b_l)$ (with $a_{s+1}=b_s$).
Observe that on each interval, the indices $l(x), u(x), l'(x)$ and $u'(x)$ are all constant. 
We have
\begin{equation*}
h(m,m')=\sum_{s=1}^\ell \int_{I_s} \min(m(x),m'(x)) \dmu(x),
\end{equation*}
and we can use the lower/upper bounds of eq.~(\ref{eq:lub}) to bound $h(m,m')$.
For a given interval $I_s=[a_s,b_s)$, we calculate
\begin{align*}
L_s(m,m') &= \int_{I_s} \min(p_{l(x)}(x),p'_{l'(x)}(x)) \dmu(x), \\
U_s(m,m') &= \int_{I_s} \max(p_{u(x)}(x),p'_{u'(x)}(x)) \dmu(x),
\end{align*}
in {\em constant time} using the cumulative distribution functions (CDFs) $\Phi_i$'s and $\Phi_i'$'s of the mixture components.
Indeed, the probability mass inside an interval $I=[a,b)$ of a component $p(x)$ (with CDF $\Phi(x)$) is simply expressed as the
difference between two CDF terms $\int_a^b p(x)\dmu(x) = \Phi(b)-\Phi(a)$.
Let $A(m,m') \eqdef \sum_{s=1}^l L_s(m,m')$ and $B(m,m')\eqdef \sum_{s=1}^l U_s(m,m')$.  It follows that
$$
A(m,m')    \leq h(m,m')  \leq  B(m,m').
$$
Notice that the above derivation applies to $H(m,m')$ as well, and
therefore
$$
A(m,m')    \leq H(m,m')  \leq  B(m,m').
$$
Thus we obtain the following lower and upper bounds of the TV:
\begin{align}
L(m,m')&\eqdef \max\{1-B(m,m'), A(m,m')-1\},\nonumber\\
%&\quad\leq \TV(m,m') \leq \nonumber\\
U(m,m')&\eqdef \min\{1-A(m,m'),B(m,m')-1\}.
\end{align}

To further improve the bounds for exponential family components,
we choose for each elementary interval $I_s$ a \emph{reference measure}
$r_s(x)=\exp\left(\theta_s^\top t(x)-F(\theta_s)\right)$, which can simply
be set to the upper envelope $p_u(x)$ over $I_s$.
Then we can bound the {\em{}density ratio}
\begin{equation*}
\frac{p_i(x)}{r_s(x)}
= \exp\left((\theta_i-\theta_s)^\top t(x)\right)
\in \left[A_s^i,B_s^i\right]
\end{equation*}
for any $p_i(x)$ in the same exponential family.
Notice that $t(x)$ is usually a vector of monomials representing a polynomial function whose bounds
can be computed straightforwardly for any given interval $[a,b)$. Therefore
\begin{equation*}
\frac{m(x)-m'(x)}{r_s(x)} =
\sum_{i=1}^k w_i \frac{p_i(x)}{r_s(x)}-\sum_{i=1}^{k'}w_i' \frac{p_i'(x)}{r_s(x)}
\end{equation*}
must lie in the range $[L_s, U_s]$, where
\begin{equation*}
L_s = \sum_{i=1}^k w_i A_s^i - \sum_{i=1}^{k'}w_i'(B_s^i)',\quad{}
U_s = \sum_{i=1}^k w_i B_s^i - \sum_{i=1}^{k'}w_i'(A_s^i)'.
\end{equation*}
Correspondingly, $\vert m(x)-m'(x)\vert/r_s(x)\in[\mu_s,\Omega_s]$.
If $L_sU_s<0$, then $\mu_s=0$, otherwise $\mu_s=\min(\vert{}L_s\vert{},\vert{}U_s\vert{})$.
$\Omega_s=\max(\vert{}L_s\vert{},\vert{}U_s\vert{})$.
Hence, we get
\begin{equation*}
\mu_s \int_{I_s} r_s(x)  \dmu(x)
\le \int_{I_s} \vert m(x)-m'(x) \vert \dmu(x) \le
\Omega_s \int_{I_s} r_s(x) \dmu(x),
\end{equation*}
and 
\begin{equation*}
\frac{1}{2} \sum_{s=1}^l \mu_s \int_{I_s} r_s(x)  \dmu(x)
\le \TV(m,m') \le
\frac{1}{2} \sum_{s=1}^l \Omega_s \int_{I_s} r_s(x) \dmu(x).
\end{equation*}
We call these bounds CELB/CEUB for combinatorial envelope lower/upper bounds.

\begin{theorem}
The total variation distance between univariate Gaussian mixtures can be deterministically approximated in
$O(n\log n)$-time, where $n=k+k'$ denotes the total number of mixture components. 
\end{theorem}

The following section describes experimental results that highlight the tightness performance of these bounds.

%%%%%%%%%
\section{Experiments}\label{sec:exp}
%%%%%%%%%

We assess the proposed TV bounds based on the following univariate GMM models 
$\mathtt{GMM}_1$, $\mathtt{GMM}_2$, $\mathtt{GMM}_3$ and $\mathtt{GMM}_4$,
which was used in \cite{NS-2017}.
We split each elementary interval into $10$ pieces of equal size so as to
improve the bound quality.
For MC (Monte Carlo) and CGQLB, we sample from both $p$ and $q$ and combine these sample sets.
Fig.~(\ref{fig:gmm}) shows the envelopes of these GMMs and the corresponding
TV. In the rightmost figure, the $x$-axis is the sample size
for MC and CGQLB, and the $y$-axis is the TV value.
The 95\% confidence interval is visualized for MC.
We can see that the proposed combinatorial bounds are quite tight
and enclose the true TV value.
Notably, the CGQLB is even tighter if the sample size is large enough.
Given the same sample size, the number of density evaluations
(computing $p(x)$ and $q(x)$ for one time) are the same for MC and CGQLB.
Therefore, instead of doing MC, one should prefer to use CGQLB
that provides a deterministic bound. The experiments are further carried on Gamma and Rayleigh mixtures, see Fig.~(\ref{fig:gmm}).

In a second set of experiments, we generate random GMMs where the means are taken from the standard
Gaussian distribution (dataset 1) or with its standard deviation increased to 5 (dataset 2).
In both cases, the precision is sampled from $\mathtt{Gamma}(5,0.2)$.
All components have equal weights.
Fig.~(\ref{fig:randgmm}) shows mean$\pm$standard deviation 
for CELB/CEUB/CGQLB against then the number of mixture components $k$.
TV (relative) shows the relative value of the bounds, which is the ratio
between the bound and the ``true'' TV estimated using $10^4$ MC samples.
CGQLB is implemented with 100 random samples drawn from a mixture of $p$ and $q$ with equal weights.
The Pinsker upper bound is based on a ``true'' KL estimated by $10^4$ MC samples.
(strictly speaking, the Pinsker bound estimated in this way is not a deterministic bound as our proposed bounds.)
We perform 100 independent runs for each $k$.
TV decreases as $k$ increases because $p$ and $q$ are more mixed.
We see that CELB and CEUB provide relatively tight bounds as compared to the Pinsker bound,
which are well enclosed by the trivial bounds $[0,1]$.
The quality of CGQLB is remarkably impressive: based on the yellow lines in
the right figures, using merely 100 random samples we can get a upper bound
which is very close to the true value of the TV.

\def\ttt{0.75}
\begin{figure*}[t]
\centering
\includegraphics[width=\ttt\textwidth]{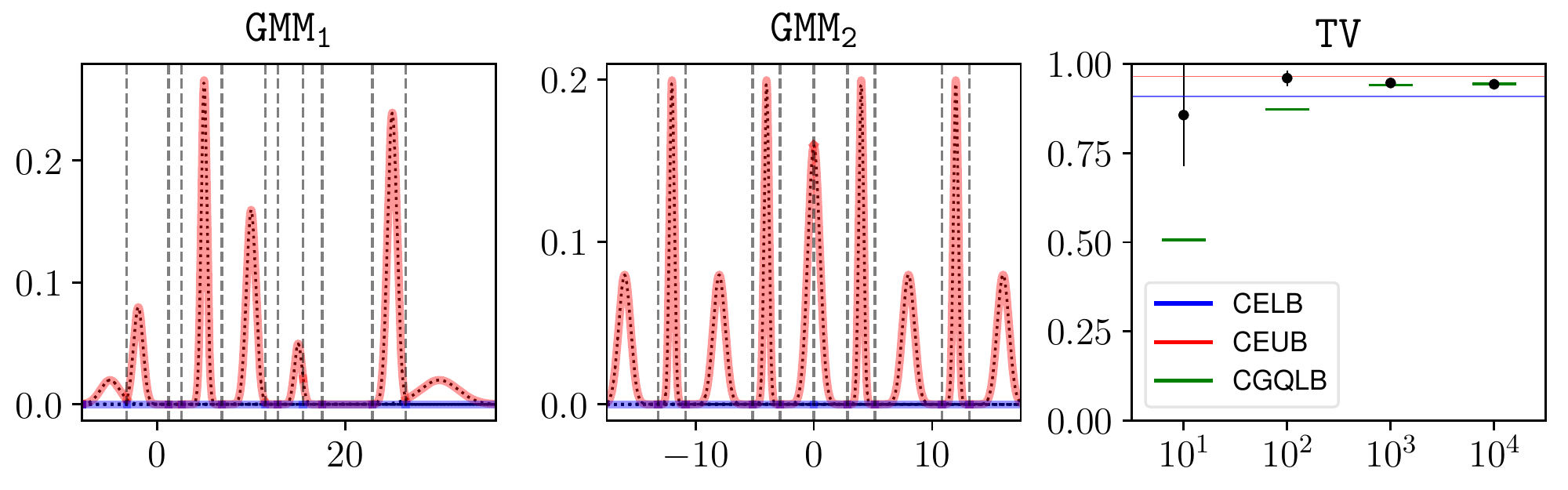}\\
\includegraphics[width=\ttt\textwidth]{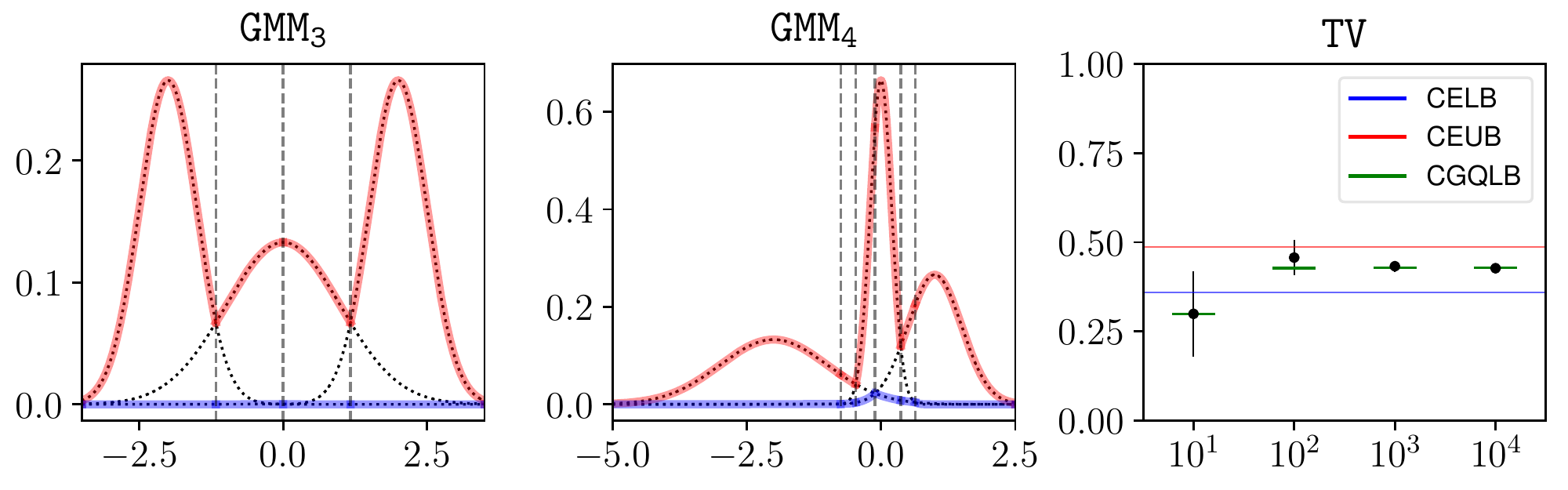}\\
\includegraphics[width=\ttt\textwidth]{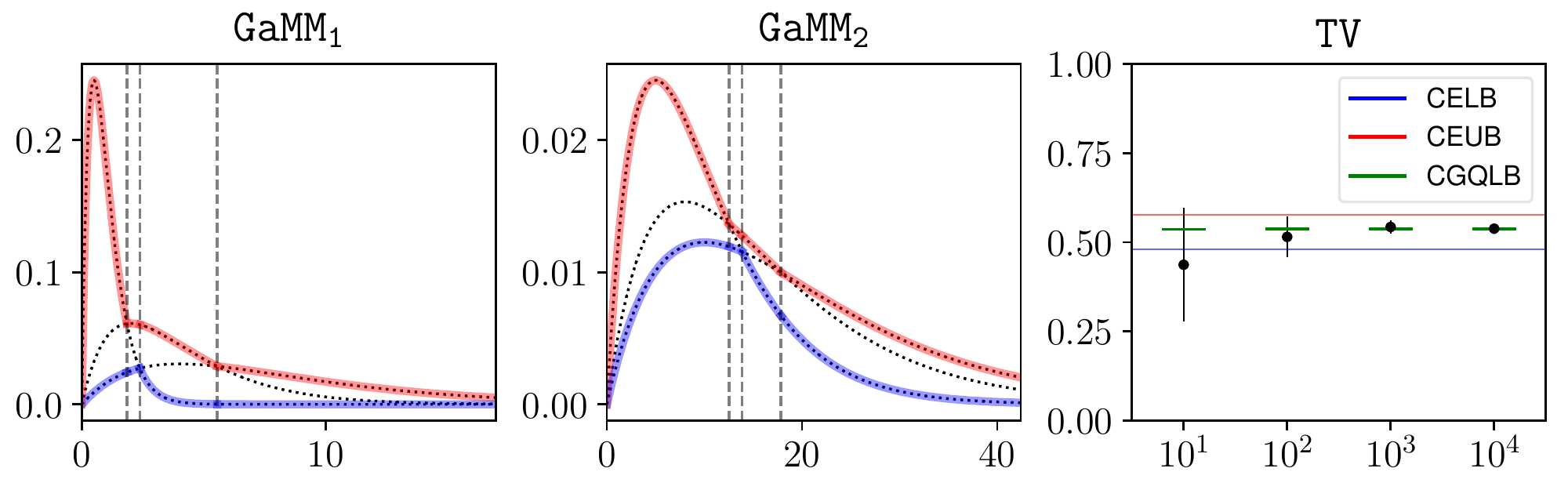}\\ 
\includegraphics[width=\ttt\textwidth]{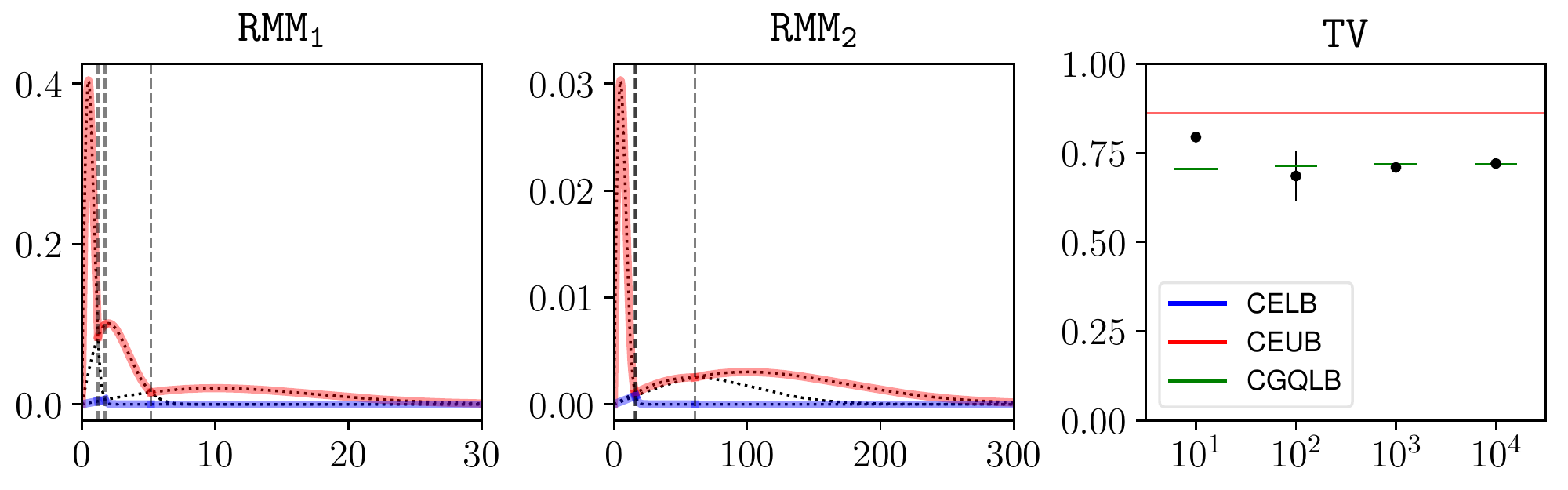}

\caption{Mixture models (from top to bottom: Gaussian, Gamma and Rayleigh) and their upper (red) and lower (blue) envelopes.
The rightmost figure shows their TV computed by 
\ding{192} MC estimation (black error bars);
\ding{193} the proposed guaranteed combinatorial bounds (blue and red lines);
\ding{194} the coarse-grained quantized lower bound (green line).}
\label{fig:gmm}
\end{figure*}

\begin{figure}[h]
\centering
\begin{subfigure}[t]{\textwidth}
\includegraphics[width=\textwidth]{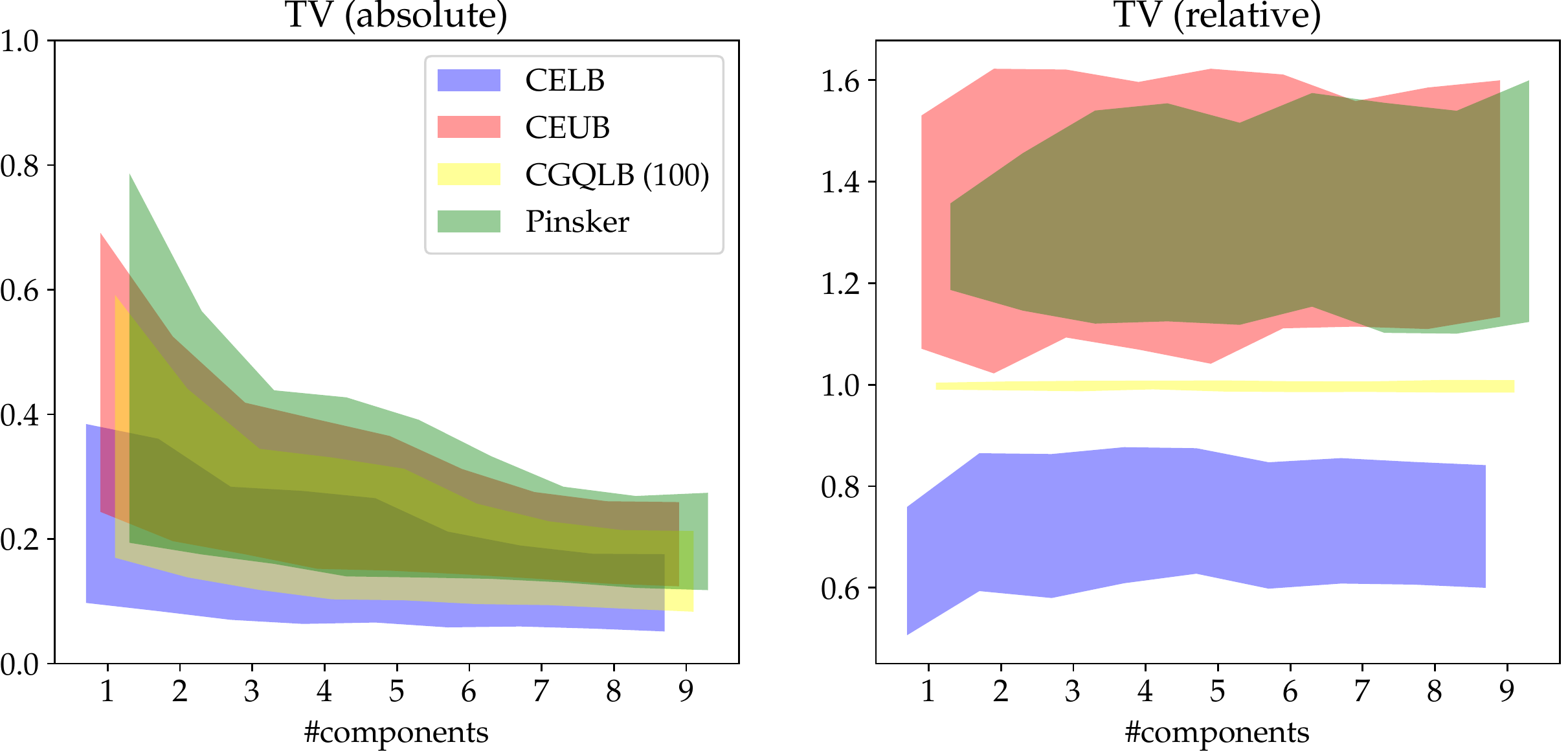}
\caption{Random Dataset 1}
\end{subfigure}
\begin{subfigure}[t]{\textwidth}
\includegraphics[width=\textwidth]{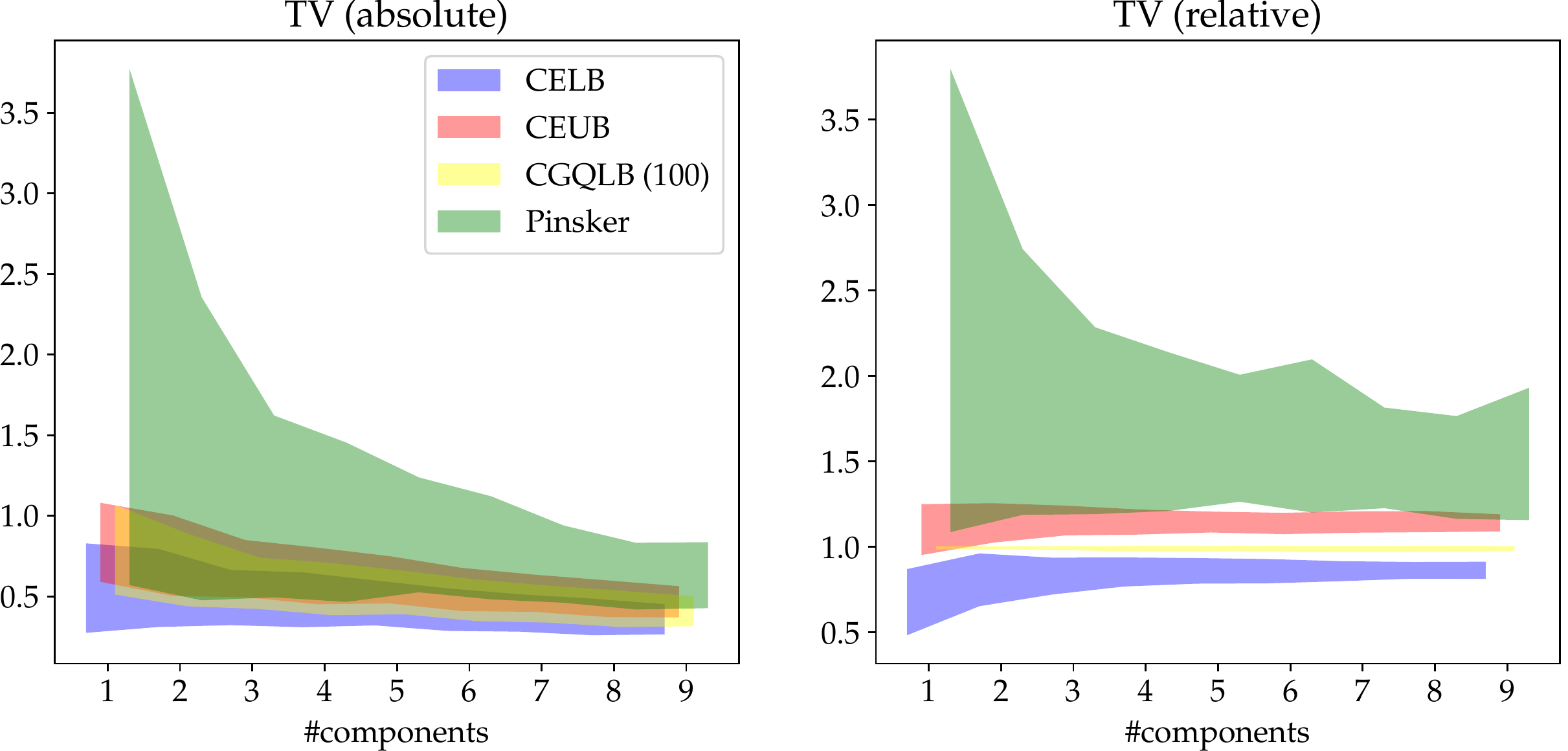}
\caption{Random Dataset 2}
\end{subfigure}
\caption{The bounds against the number of components
on random GMMs with equal component weights. 100 different pairs of GMMs are generated for each
configuration. The mean$\pm$std. is presented as a colored band for each method.
TV (absolute) shows the absolute values of the bounds. TV (relative) shows the
value of the bounds divided by the ``true'' TV estimated by MC sampling ($10^4$ samples).}\label{fig:randgmm}
\end{figure}

%\begin{figure*}[t]
%\includegraphics[width=\textwidth]{tv_gamm1_gamm2.pdf}\\ 
%\includegraphics[width=\textwidth]{tv_rmm1_rmm2.pdf}
 %
%\caption{(Top) Two Gamma MM models and their upper (red) and lower (blue) envelopes.
%The rightmost figure shows their TV computed by 
%\ding{192} MC estimation (black error bars);
%\ding{193} the proposed guaranteed combinatorial bounds (blue and red lines);
%\ding{194} the coarse-grained quantized lower bound (green line).}\label{fig:gammamm}
%(Bottom) Two Rayleigh mixtures
%\end{figure*}

%%%%%%%%%%
\section{Conclusion and discussion}\label{sec:concl}
%%%%%%%%%%

We described novel deterministic lower and upper bounds on the total variation distance between univariate mixtures, and demonstrated their effectiveness for Gaussian, Gamma and Rayleigh mixtures. 
This task is all the more challenging since the TV value is falling in the range $[0,1]$, and that the designed bounds should improve over these naive bounds.
A first proposed approach relies on the information monotonicity~\cite{IG-2016} of the TV to design a lower bound (or a series of nested lower bounds), and can be extended to arbitrary $f$-divergences.
A second set of techniques uses tools of computational geometry to compute mixture component upper and lower geometric envelopes of their weighted component univariate distributions, and retrieve from these decompositions both  Combinatorial Envelope  Lower and Upper
Bounds (CELB/CEUB).
All those methods certify {\em deterministic} bounds, and are therefore recommended over the traditional Monte Carlo stochastic approximations that has no deterministic guarantee (although being consistent asymptotically).

Finally, let us discuss the role of {\em generalized TV distances} in $f$-divergences:
$f$-divergences are statistical separable divergences which admit the following {\em integral-based representation}~\cite{Liese-2006,GenPinsker-2009,InfoDivRisk-2011,Sason-2016}:
\begin{eqnarray*}
I_f^*(p:q) &=& \int_\calX q(x) f\left(\frac{p(x)}{q(x)}\right)\dmu(x),\\
I_f^*(p:q) &=& \int_0^1 {\frac{1}{u^3}f''(\frac{1-u}{u})}  \TV_u(p:q)\du,\\
\TV_u(p:q) &\eqdef & I_{f_u}^*(p:q),\\
f_u(t) &\eqdef & \min\{u,1-u\}-\min\{1-u,ut\}.
\end{eqnarray*}
Here, we have $I_f^*(p:q)=I_f(q:p)=I_{f^\diamond}(p:q)$ for $f^\diamond(u)=uf(1/u)$, see~\cite{Csiszar-2004}.
$\TV_u$ are generalized (bounded) total variational distances, and our deterministic bounds can be extended to these $\TV_u$'s. 
However, note that $I_f$ may be infinite (unbounded) when the integral diverges.

Code for reproducible research is available at \url{https://franknielsen.github.io/BoundsTV/index.html} 

\vspace{2em}
 
\bibliographystyle{IEEEbib}
\bibliography{LowerUpperBoundBIB}

\end{document}